\documentclass[letterpaper, 10 pt, journal, twoside]{ieeetran}

\IEEEoverridecommandlockouts                              

\usepackage{amsmath,amsthm}    
\usepackage{amssymb}    
\usepackage{units}      
\usepackage{comment}    
\usepackage{color}      

\usepackage[table]{xcolor}
\usepackage{hyperref}
\usepackage{units}

\usepackage{graphicx}
\usepackage{epstopdf}

\usepackage{cite}
\usepackage{enumerate}

\graphicspath{{Figures/}}

\newcommand{\mat}[1]{\boldsymbol{#1}}
\renewcommand{\vec}[1]{\boldsymbol{#1}}

\usepackage{soul}

\newcommand{\revcomment}[2]{\textcolor{red}{#2}$^{\textcolor{blue}{\# #1}}$} 

\renewcommand{\revcomment}[2]{#2}  

\newcommand{\mr}[1]{\mathrm{#1}}
\newcommand{\T}{^{\mathop{\mathrm{T}}}}

\makeatletter
\renewcommand*\env@matrix[1][\arraystretch]{%
  \edef\arraystretch{#1}%
  \hskip -\arraycolsep
  \let\@ifnextchar\new@ifnextchar
  \array{*\c@MaxMatrixCols c}}
\makeatother

\usepackage[ruled,vlined]{algorithm2e}
\SetKwComment{Comment}{/* }{ */}

\newtheorem*{definition}{Definition}

\newtheorem{lemma}{Lemma}[section]
\newtheorem*{thm}{Theorem}

\newtheorem{remark}{Remark}[section]

\title{Connecting Gaits in Energetically Conservative Legged Systems}

\author{Maximilian Raff$^{1}$, Nelson Rosa Jr.$^{1}$, and C. David Remy$^{1}$

\thanks{Manuscript received: 4 February 2022; accepted 8 June 2022, This letter was recommended for publication by Associate Editor E. M. Hoffman and Editor L. Pallottino upon evaluation of the reviewers' comments.
This work was supported through the International Max Planck Research School for Intelligent Systems (IMPRS-IS) for Maximilian Raff and an Alexander von Humboldt fellowship to Nelson Rosa.  Additional funding for C. David Remy was provided by the Carl Zeiss Foundation. }
\thanks{$^{1}$The authors are with the Institute for Nonlinear Mechanics, University of Stuttgart, D-70569 Stuttgart, Germany.
        {\tt\small \{raff, nr, remy\}@inm.uni-stuttgart.de}}%
\thanks{Digital Object Identifier (DOI): see top of this page.}
}

\begin{document}

\begin{titlepage}
\begin{center}
This preprint has been accepted for publication in IEEE Robotics and Automation Letters.\\
DOI: 10.1109/LRA.2022.3186500\\
IEEE Explore: https://ieeexplore.ieee.org/document/9809817\\[5mm]

Please cite the paper as:
M. Raff, N. Rosa and C. D. Remy, "Connecting Gaits in Energetically Conservative Legged Systems," in IEEE Robotics and Automation Letters, vol. 7, no. 3, pp. 8407-8414, July 2022.\\[15mm]

\end{center}
\end{titlepage}

\maketitle

%
\begin{abstract}
In this work, we present a nonlinear dynamics perspective on generating and connecting gaits for energetically conservative models of legged systems.  
In particular, we show that the set of conservative gaits constitutes a connected space of locally defined 1D submanifolds in the gait space.
These manifolds are coordinate-free parameterized by energy level.

We present algorithms for identifying such families of gaits through the use of numerical continuation methods, generating sets and bifurcation points. 
To this end, we also introduce several details for the numerical implementation. 
Most importantly, we establish the necessary condition for the Delassus’ matrix to preserve energy across impacts.

 An important application of our work is with simple models of legged locomotion that are often able to capture the complexity of legged locomotion with just a few degrees of freedom and a small number of physical parameters.  
We demonstrate the efficacy of our framework on a one-legged hopper with four degrees of freedom.
\end{abstract}


%
\section{Introduction}
\label{sec:Intro}
%
%

\IEEEPARstart{S}{implistic} conservative models of legged locomotion, in which no energy is lost during a stride, are a powerful tool for both the analysis of human and animal gaits in nature and the design and control of legged robots \mbox{\cite{mcmahon1987, Geyer2018, Koolen2012, Bhounsule2014}}.  With just a few degrees of freedom and a small number of physical parameters, these models can accurately predict the preferred locomotion patterns of humans \cite{Kuo_2001} and provide useful templates for energy-efficient robot motions \cite{Collins_2005}.

Despite the benefits of such models, the field is still lacking a unified approach that systematically takes advantage of the conservative nature of these models to identify and characterize the different types of periodic motions available.  This becomes even more important given that the same model can exhibit multiple modes of locomotion (e.g., walking, hopping, and running).  To the best of our knowledge, past works have only developed results for specific conservative models and gait type \cite{gan2018, oconnor2009, garcia1998simplest, merker2015} and not a class of energetically conservative systems with hybrid dynamics and multiple modes of locomotion.  The goal of this paper is to create a mathematical framework rooted in the theory of hybrid dynamical systems and nonlinear dynamics to model, classify, and create periodic motions for energetically conservative models (ECMs) of legged systems.

To this end, we generalize the methodology introduced in~\cite{gan2018} and carefully embed it into a mathematical framework for general ECMs of legged systems. 
We prove that families of gaits exist for such systems and highlight the role of energy in providing a coordinate-free parameterization for these families.
In order to make the approach practical, we present algorithms for identifying families of gaits through the use of numerical continuation methods and introduce a number of details for their implementation. 
Among others, these details include projecting the state space to the subspace of periodic motions, establishing the necessary condition for the Delassus’
matrix to preserve energy across impacts, introducing the
use of additional (holonomic) constraints to avoid singular
dynamics, embedding the conservative system in a one-parameter family of dissipative systems and transitioning from an event-driven formulation to a time-based formulation.

This paper can be considered to be a direct extension of~\cite{gan2018} which showed that a simple model exhibits all common bipedal gaits and that these form continuous families of gaits in the biped's space of trajectories.  
These periodic motions all emerged from a one-dimensional (1D) family of hopping-in-place gaits.  
Other gaits, such as walking and running, were connected to these through a series of bifurcations.
Furthermore, our work builds upon the one-parameter families of periodic orbits in smooth ECMs as they are the main subject in~\cite{sepulchre1997} and \cite{albu2020}.
While \cite{sepulchre1997} provides conditions for the existence of this family, \cite{albu2020} revisits concepts of so-called Nonlinear Normal Modes (NNMs) that aim to find analytic expressions of invariant lower-dimensional submanifolds.
Herein, NNMs are explicitly parameterized representations of 1D manifolds that emanate from exploiting the system's state dependencies inflicted by the conservation of energy.

In the remainder of this paper, we first introduce the mathematical theory for ECMs (Section~\ref{sec:theory}) before discussing a numerical algorithm for the automated search for gaits (Section~\ref{sec:implementation}).
The example application of a one-legged hopper then further illustrates these concepts (Section~\ref{sec:example}).

%
\section{Theory}
\label{sec:theory}

\subsection{Dynamics of Legged Systems}
\revcomment{1.11}{In our work, we consider rigid body systems subject to contact without sliding, as they are commonly used to model legged robotic systems.
An important restriction is that we limit ourselves to ECMs and periodic motions with a particular footfall sequence; for example, to either running or walking.
The state of such a system is given by the vector \mbox{$\vec x = \left( \vec q, \dot{\vec{q}}\right)\in T\mathcal{Q}\subset \mathbb{R}^{2n_\mathrm{q}}$}, where $n_\mathrm{q}$ is the number of its degrees of freedom and \mbox{$T\mathcal{Q}$} is the tangent bundle of the configuration space $\mathcal{Q}\subset \mathbb{R}^{n_\mathrm{q}}$.
In the following, we heavily rely on the concepts, assumptions, and notation from \cite{grizzle2014models}.
We refer to a motion within a \revcomment{1.3}{persistent contact configuration as a \emph{phase} $i$.}
These phases are executed in a fixed, repeating order $1\to2\to\dots\to m\to 1$.}
\revcomment{1.11}{Adopting the notation of \cite{grizzle2014models}, the hybrid model is written as
\begin{align*}
\Sigma: \left\{ 
\arraycolsep=1.0pt
\begin{array}{lll} \mathcal{X} &= {\{\mathcal{X}_i\}}_{i=1}^m :& \mathcal{X}_i=\{\vec x\in T\mathcal{Q}: \vec g_{i}(\vec q)=\vec{0} \}\\[2mm]
\mathcal{F} &= {\{\vec{f}_i\}}_{i=1}^m :& \dot{\vec{x}}= \vec f_i(\vec x), ~\vec x\in \mathcal{X}_i
\\[2mm]
\mathcal{E} &= {\{\mathcal{E}_i^{i+1}\}}_{i=1}^m :& \mathcal{E}_i^{i+1}=\left\{\vec x\in \mathcal{X}_i\bigg\vert \begin{matrix}e_i^{i+1}(\vec{x})=0,\\\dot e_i^{i+1}(\vec{x})<0\end{matrix}\right\}\\[4mm]
\mathcal{D} &= {\{\vec{\Delta}_i^{i+1}\}}_{i=1}^m :& \vec{x}^{+}= \vec\Delta_i^{i+1}(\vec x^-),\\
  &  & \vec x^-\in \mathcal{E}_i^{i+1},\vec x^+\in \mathcal{X}_{i+1}
\end{array}
\right.,
\end{align*}
where the codimension-one submanifold $\mathcal{E}_i^{i+1}$ determines a transition from phase $i$ to phase $i+1$ with the reset map~$\vec{\Delta}_i^{i+1}$.
The representation of the autonomous flow $\vec f_i$ in phase $i$ reflects the assumption of independent scleronomous constraints \revcomment{1.3}{$\vec g_i:\mathcal{Q}\to \mathbb{R}^{n_{\lambda_i}}$} that allows us to uniquely solve for contact forces \revcomment{1.3}{$\vec{\lambda}_i\in \mathbb{R}^{n_{\lambda_i}}$} (Theorem 5.1 \cite{brogliato2016}).
That is, the constraint Jacobian \mbox{$\mat{W}_i(\vec q)\T:=\partial \vec{g}_i/\partial \vec{q}$} in the differential-algebraic equation
\begin{subequations}\label{eq:DAE}
\begin{align}
\mat{M}(\vec{q})\ddot{\vec{q}}&=\vec k(\vec q)+\vec h(\vec q, \dot{\vec{q}})+ \mat W_i(\vec q) \vec \lambda_i,\label{eq:DAE1}\\
\vec g_i(\vec q)&=\vec{0},\label{eq:DAE2}
\end{align}
\end{subequations}
is full rank for all motions in phase $i$.
The mass matrix $\mat M$, elastic forces $\vec k$ and gravitational, centrifugal, and coriolis forces~$\vec h$ are derived from the kinetic energy \mbox{$E_\mathrm{kin}: T\mathcal{Q} \to \mathbb{R}$} and potential energy \mbox{$E_\mathrm{pot}: \mathcal{Q} \to \mathbb{R}$ of the system}.
Note that we exclude non-potential forces in equation~\eqref{eq:DAE1}, since~$\Sigma$ is assumed to be energetically conservative.}
\revcomment{1.11}{With~\mbox{$\vec x^-=\left( \vec q^-, \dot{\vec{q}}^-\right)\in\mathcal{E}_{i}^{i+1}$} and \mbox{$\vec x^+=\left( \vec q^+, \dot{\vec{q}}^+\right)\in\mathcal{X}_{i+1}$}, the reset map~$\vec{\Delta}_i^{i+1}$ does only alter the generalized velocities:
\begin{equation}\label{eq:DiscreteMap}
    \vec{x}^{+}= \vec\Delta_i^{i+1}(\vec x^-) = \begin{bmatrix}[1.2]\vec{q}^-\\\vec{P}_{i+1}(\vec q^-)\dot{\vec{q}}^-\end{bmatrix}.
\end{equation}
Since we consider plastic collisions with \mbox{$\mat W_{i+1}\T \dot{\vec{q}}^+=\vec 0$}, the reset map is given by~\mbox{$\vec{P}_{i+1}=\mat I-\mat M^{-1}\mat W_{i+1}\mat G_{i+1}^{-1}\mat W_{i+1}\T$}, where $\mat I$ is the identity matrix.}
In the field of nonlinear mechanics, the matrix \revcomment{1.11}{\mbox{$\mat G_{i+1} = \mat W_{i+1}\T \mat M^{-1} \mat W_{i+1}$}} is known as the Delassus' matrix of contact configuration \revcomment{1.11}{$i+1$} \cite{brogliato2016}.
It describes the inertial coupling in the active constraint space\footnote{$-\mat G_j^{-1}$ is called the constrained contact inertia tensor in \cite{johnson2016hybrid}.}.

\revcomment{1.11}{
As in \cite{grizzle2014models}, we also state the hybrid model as a tuple \mbox{$\Sigma =\left(\mathcal{X},\mathcal{E},\mathcal{D},\mathcal{F}\right)$}.
Furthermore, we take on the assumptions from \cite{grizzle2014models} to yield a well-posed hybrid model $\Sigma$. 
Some of these assumptions state that \revcomment{1.9}{$\Sigma$ is $C^1$}, a motion of $\Sigma$ is transversal to $\mathcal{E}_{i}^{i+1}$ if its closure intersects $\mathcal{E}_{i}^{i+1}$, and a solution through a domain $i$ must have a non-zero duration.
Hence, they avoid grazing contacts and chattering. 
Other assumptions are already built in the hybrid model $\Sigma$, such as a fixed cyclic phase sequence or scalar \emph{event functions} $e_i^{i+1}$.
\revcomment{1.3}{The latter excludes motions with simultaneous touch-downs and lift-offs, e.g., bipedal hopping or quadrupedal trotting.}
Please refer to \cite{grizzle2014models} and the references therein for a detailed overview of the required assumptions to hold for $\Sigma$.}

\revcomment{1.11}{
The \emph{phase flow} \mbox{$\vec \varphi_i: \mathbb{R}_{\geq 0} \times \mathcal{X}_i \rightarrow \mathcal{X}_i$} describes a solution to equations \eqref{eq:DAE} and thus, the motion through a phase~$i$ starting from an initial condition $\vec x_{0,i}\in\mathcal{X}_i$.
As in \cite{grizzle2014models}, we also define the phase-$i$ \emph{time-to-impact function} \mbox{$t_{\mr{I},i}(\vec{x}_{0,i})=:\inf\{t\geq 0 \vert \vec \varphi_i(t,\vec{x}_{0,i})\in\mathcal{E}_i^{i+1}\}$} if there exists a time~$t$ such that $\vec \varphi_i(t,\vec{x}_{0,i})\in\mathcal{E}_i^{i+1}$.
We start and end the cycle $1\to2\to\dots\to m\to 1$ within phase $i=1$ and denote the initial state to $\Sigma$ as \mbox{$\vec x_0:=\vec x_{0,1}$}.
With the assumptions in~\cite{grizzle2014models}, the \emph{hybrid flow} of a complete cycle \mbox{$\vec \varphi: \mathbb{R}_{\geq 0} \times \mathcal{X}_1 \rightarrow \mathcal{X}_1$} is recursively defined as
\begin{align}
    \vec x(t) &:= \vec \varphi(t,\vec x_0) = \vec{\varphi}_1\left(t-t_{\mr{I}},\vec{x}_{0,m+1}\right),\label{eq:flow}\\
    \vec{x}_{0,i+1} &= \vec{\Delta}_i^{i+1}\circ\vec{\varphi}_{i}\left(t_{\mr{I},i}\circ\vec{x}_{0,i},\vec{x}_{0,i}\right), ~
    i=1,\dots,m,\label{eq:recursion}
\end{align}
where $\vec{\Delta}_m^{m+1}=\vec{\Delta}_m^{1}$ and $0\leq t-t_{\mr{I}}< t_{\mr{I},1}\circ\vec{x}_{0,m+1}$ with the accumulated impact times \mbox{$t_{\mr{I}}:=\Sigma_{i=1}^m t_{\mr{I},i}\circ\vec{x}_{0,i}$}.
To further simplify the following statements, let us define the interval \mbox{$\mathcal{I}:=[t_{\mr{I}},t_{\mr{I}}+t_{\mr{I},1}\circ\vec{x}_{0,m+1})$}.
Herein, $t-t_{\mr{I}}\in\mathcal{I}$ is the time spent in the last phase of a cycle.}
\revcomment{1.11}{
\begin{remark}
In contrast to the \emph{Poincaré return map} in \cite{grizzle2014models}, the initial condition $\vec{x}_{0}$ of the hybrid flow in equation \eqref{eq:flow} can be chosen arbitrarily in the domain $\mathcal{X}_1$ and does not necessarily lie in the image of $\vec{\Delta}_m^1$. 
This definition of the hybrid flow enables us to directly relate to known properties of autonomous nonlinear dynamical systems.
It will, however, require the construction of an additional event-like \emph{anchor constraint} later on.
\end{remark}
}
\revcomment{1.9}{With the aforementioned assumptions from \cite{grizzle2014models}, the fundamental solution matrix
\begin{equation}
\label{eq:fundamentalMatrixHybrid}
    \mat{\Phi}\left(t,\vec{x}_0\right) = \frac{\partial \vec{\varphi}\left(t,\vec{x}\right)}{\partial \vec{x}}\bigg|_{\vec x=\vec{x}_0} \in \mathbb{R}^{2n_\mr{q} \times 2n_\mr{q}}
\end{equation}
is well-defined for any $t\in\mathcal{I}$ \cite{muller1995, ivanov1998}.}

\subsection{Periodic Solutions in Energetically Conservative Hybrid Dynamical Systems}
\revcomment{1.11}{
The total energy of the hybrid model $\Sigma$ is given by \mbox{$E(\vec x) =E_\mathrm{kin}(\vec q,\dot{\vec q})+E_\mathrm{pot}(\vec q)$}.
\begin{definition}[Energetically Conservative
Model]
The hybrid system $\Sigma$ is an energetically conservative model (ECM) if 
\begin{enumerate}[{Df}1]
    \item all forces in the continuous dynamics of equation \eqref{eq:DAE1} are conservative forces and
    \item for all reset maps $\vec{x}^+= \vec \Delta_i^{i+1}(\vec x^-)$ it holds \mbox{$E(\vec x^+) = E(\vec x^-)$}. This implies \mbox{$E_\mathrm{kin}(\vec x^+) = E_\mathrm{kin}(\vec x^-)$}, since the discrete dynamics, with $\vec q^+=\vec q^-$, do not change the value of $E_\mathrm{pot}$; i.e., $E_\mathrm{pot}(\vec q^+) = E_\mathrm{pot}(\vec q^-)$.
\end{enumerate}
\end{definition}
\revcomment{1.11}{
The definition of an ECM implies that for any $\vec x_0$ its total energy $E$ is invariant under the hybrid flow $\varphi(t,\vec{x}_0)$ for all times $t\in\mathcal{I}$.
}
%
\begin{definition}[Hybrid Periodic Flow]
\revcomment{1.11}{A hybrid flow defined by equation \eqref{eq:flow} is \emph{periodic}, if there exists a period time \mbox{$T\in\mathcal{I}$}, such that 
\begin{equation}\label{eq:periodicity}
    \vec{\varphi}(T,\vec{x}_0) - \vec{x}_0 = \vec{0}.
\end{equation}}
\end{definition}
}

\begin{definition}[Monodromy Matrix]
The local linearization of a periodic solution \mbox{$\mat{\Phi}_T:=\mat{\Phi}\left(T,\vec{x}_0\right)$} is called the monodromy matrix.
\end{definition}

The monodromy matrix is an important tool to study the stability and local existence of periodic flows (Chapter~7.1.1~\cite{leine2013dynamics}).
For \revcomment{1.11}{autonomous} ECMs, it holds that:
\revcomment{1.11}{
\begin{align}
    \mat{\Phi}_T \vec{f}_1(\vec{x}_0) &= \vec{f}_1(\vec{x}_0),\label{prop:FreedomOfPhase}\\
    \nabla E\left(\vec{x}_0\right)\T \mat \Phi_T &= \nabla E(\vec{x}_0)\T.\label{eq:conserveEnergy}
\end{align}}
Equation \eqref{prop:FreedomOfPhase} is the well known freedom of phase in autonomous systems, as any disturbance along the flow will remain on the same periodic motion in $T\mathcal{Q}$ (Theorem~2~\cite{sepulchre1997}).
\revcomment{1.11}{Furthermore, since the total energy is flow-invariant: $E(\vec \varphi(t,\vec{x}_0))=const.=\bar{E}$, this yields the property in equation \eqref{eq:conserveEnergy} (Chapter~2.4.~\cite{sepulchre1997}).}

\revcomment{1,11}{
\begin{lemma}\label{remark:perbVec}
Outside of an equilibrium, where $\nabla E(\vec{x}_0)$ and $\vec f_1(\vec{x}_0)$ are non-zero for a mechanical system, these vectors are also perpendicular.
\end{lemma}
\begin{proof}
Since the energy $E(\vec \varphi_1(t,\vec{x}_0))$ in phase $i$ is constant for all \mbox{$t\in[0,t_{\mr{I},1}(\vec{x_0}))$}, this implies:
\begin{equation}\label{eq:perpVectors}
    \frac{\mr{d}}{\mr{d}t}E(\vec \varphi_1(t,\vec{x}_0))\bigg\vert_{t=0} = \nabla E(\vec{x}_0)\T \vec{f}_1(\vec{x}_0) =0.
\end{equation}
\end{proof}
}
\subsection{Connected Components of Energetically Conservative Gaits}

The purpose of this work is to show connections between different periodic motions that we will refer to as different \emph{gaits}.
To eliminate the freedom-of-phase that is inherent to any autonomous system, we introduce an anchor constraint to further specify the solution that constitutes a specific gait:

\begin{definition}[Gait]
A gait is a periodic solution that also fulfills the anchor constraint \revcomment{1.11}{$a(\vec{x}_0)=0$}, where $a:\mathcal{X}_1\to\mathbb{R}$ is a smooth function for which the transversality condition \revcomment{1.11}{$\nabla a(\vec{x}_0)\T \vec{f}_1(\vec x_0)\neq 0$} holds.
\end{definition}

\revcomment{1.11}{\begin{thm}[Family of Gaits]
In the vicinity of a energetically conservative gait there exist neighboring gaits.
\begin{proof}
Due to the periodicity, it must hold:
\begin{align}
    a(\vec{x}_0)=a(\vec{\varphi}(T(\vec{x}_0),\vec{x}_0)) = 0,
\end{align}
where we abuse the notation of the period $T=T(\vec{x}_0)$ to indicate its general dependency on $\vec{x}_0$. 
Using the implicit function theorem, we get:
\begin{align}
    \frac{\partial T}{\partial \vec{x}_0} &= -\frac{\nabla a(\vec{x}_0)\T}{\nabla a(\vec{x}_0)\T \vec{f}_1(\vec x_0)} \mat{\Phi}_T.\label{eq:dTdx}
\end{align}
To explore neighboring gaits, we perturb the initial state of the periodic solution \eqref{eq:periodicity} by an infinitesimal $\delta \vec x$:
\begin{equation}\label{eq:periodicityPerturbed}
    \vec{\varphi}\left(T(\vec{x}_0+\delta \vec{x}),\vec x_0+\delta \vec{x}\right)-\left(\vec{x}_0 +\delta \vec{x}\right) = \vec{0}.
\end{equation}
A first-order approximation of equation \eqref{eq:periodicityPerturbed} yields
\begin{align}
    &\underbrace{\vec{\varphi}(T(\vec{x}_0),\vec{x}_0)-\vec{x}_0}_{\overset{\eqref{eq:periodicity}}{=}\vec{0}} + \vec{f}_1(\vec x_0)\frac{\partial T}{\partial \vec{x}_0}\delta \vec x +\mat{\Phi}_T\delta \vec{x} -\delta \vec{x} = \vec{0},\notag\\
    &\overset{\eqref{eq:dTdx}}{\Rightarrow} \Bigg(\underbrace{\mat{\Phi}_T -\mat{I}-\frac{\vec{f}_1(\vec x_0)\nabla a(\vec{x}_0)\T}{\nabla a(\vec{x}_0)\T \vec{f}_1(\vec x_0)} \mat{\Phi}_T}_{=:\mat{D}}\Bigg)\delta\vec{x}  = \vec 0.\label{eq:firstOrderNoFreedom}
\end{align}
As the anchor constraint removes the freedom of phase, 
\mbox{$\delta \vec{x}=\vec{f}_1(\vec x_0) \delta t$}, with $\vert \delta t\vert \ll 1$, does not solve equation~\eqref{eq:firstOrderNoFreedom}, since
$\mat D\vec{f}_1(\vec x_0) = -\vec{f}_1(\vec x_0)$. 
Yet, because of equation \eqref{eq:conserveEnergy} and $\nabla E(\vec{x}_0)\T\vec{f}_1(\vec x_0) = 0 $ (Lemma \ref{remark:perbVec}), $\nabla E(\vec{x}_0)$ is in the kernel of $\mat D\T$. 
This implies that $\mr{dim}(\mr{ker}(\mat D))\geq 1$ and thus, the existence of a nontrivial direction $\delta \vec{x}$ which must be linearly independent of $\vec{f}_1(\vec x_0)$. 
\end{proof}
\end{thm}}
\begin{remark}\label{remark:PeriodTime}
Unlike in linear systems, the period $T$ can change locally in nonlinear systems. This information is lost in the linearization $\mat{\Phi}_T$. However, imposing an anchor constraint on equation \eqref{eq:periodicity} implicitly defines a Poincaré section \cite{leine2013dynamics}, which associates $T$ with the initial states $\vec{x}_0$.
\end{remark}
\begin{remark}
This Proprosition is an extension of Theorem~4 in \cite{sepulchre1997} that proves that for smooth conservative dynamics, orbits are dense in the state space $T\mathcal{Q}$.
\end{remark}
\begin{remark}
What was shown here for energy can be extended to other flow invariant functions\footnote{These so-called first integrals are considered in \cite{munoz2003continuation} for smooth systems.}. 
For example, in some mechanical systems, linear or angular momentum may be conserved.
The existence of such invariants can then lead to additional left eigenvectors as in equation \eqref{eq:conserveEnergy} and hence in the kernel of $\mat{D}\T$ in equation \eqref{eq:firstOrderNoFreedom}.
\end{remark}

We propose to parameterize the resulting families of connected gaits by energy level $\bar{E}$.
While other parameterizations are possible (e.g., using a state variable, such as speed \cite{gan2018}), $\bar{E}$ gives a more general coordinate-free parameterization for ECMs, since gaits are inherently constrained to an equipotential surface (Lemma \ref{remark:perbVec}).
%
%
%
This parameterization is reflected in:
\begin{equation}
    \label{eq:uniqueSolution}
    \vec{r}_{\bar{E}}(\vec{x}_0,T) := 
    \begin{bmatrix}[1.2]
     \varphi(T,\vec{x}_0) - \vec{x}_0  \\
      a(\vec{x}_0) \\
      E(\vec{x}_0) - \bar{E}  
    \end{bmatrix}
    =
    \vec{0},
\end{equation}
with its derivative
\revcomment{1.11}{
\begin{equation*}
    \begin{aligned}
    \mat{R}_{\bar{E}}(\vec{x}_0,T) := 
      \dfrac{\partial \vec{r}_{\bar{E}}}{\partial [T~\vec{x}_0\T]}= \begin{bmatrix}[1.2]
      \mat{\Phi}_T-\mat{I}& \vec{f}_1(\vec x(T))\\
       \nabla a(\vec{x}_0)\T & 0\\
       \nabla E(\vec{x}_0)\T & 0
    \end{bmatrix}.
    \end{aligned}
\end{equation*}
}
%
The set of all solutions (with admissible flow) to equation~\eqref{eq:uniqueSolution} for all possible energy levels $\bar{E}$ constitutes the \emph{gait space} \mbox{$\mathcal{G}=\{(\vec{x}_0,T,\bar{E})\in \mathcal{X}_1\times \mathbb{R}_{>0}\times \mathbb{R}~:~ \vec{r}_{\bar{E}}(T,\vec{x}_0) = \vec{0}\}$}.
\begin{definition}[Regular Point]
We call a solution $\vec{z}^\ast$ of an implicit function \revcomment{1.11}{$\vec{F}: \mathbb{R}^j \to \mathbb{R}^k$ with $\vec{F}(\vec z^\ast)=\vec{0}$} a \emph{regular point} if $(\partial F/\partial \vec z)\vert_{\vec{z}=\vec{z}^\ast}$ has maximum rank.
\end{definition}

\begin{figure}[t]
    \centering
    \includegraphics{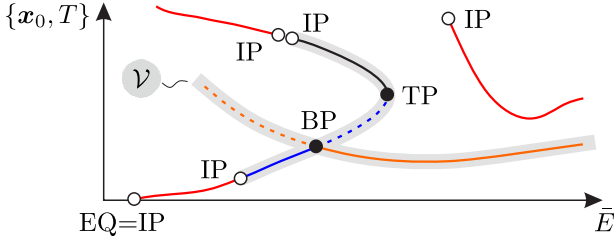}
    \caption{Different generators are connected by bifurcation (BP) and turning (TP) points and constitute the connected component $\mathcal{V}$ of the gait space $\mathcal{G}$. 
    Isolated generators and generators (red) that only connect to inadmissible points~(IP), including equilibira (EQ), are disjoint. 
    Hence, they are part of different connected components.}
    \label{fig:Generators}
\end{figure}

While $\mat{R}_{\bar{E}}$ has full rank, there exists a set of regular points $(\vec{x}_0, T, \bar{E})$ that form a locally defined \mbox{1D} submanifold~$\mathcal {M}\subseteq~\mathcal{G}$.  
Since each point of $\mathcal{M}$ represents a periodic motion, $\mathcal{M}$ is also called a generator for a two-dimensional invariant set of solutions in the state space $T\mathcal{Q}$ \cite{albu2020}.
\begin{definition}[Generators] ~
\begin{enumerate}[1)]
    \item A set $\mathcal{S}\subseteq \mathcal{G}$ is path-connected if for any two points $a,b\in\mathcal{S}$, there exists a continuous function $\gamma:[0,1]\to \mathcal{G}$ such that $\gamma(0)=a$ and $\gamma(1)=b$.
    \item A set $\mathcal{M}\subseteq~\mathcal{G}$ is called a generator if it is path-connected and all points $a\in \mathcal{M}$ are regular.
\end{enumerate}
\end{definition}
%
Generators can border to a point $(\vec{x}_0, T, \bar{E})\notin \mathcal{G}$ which \revcomment{1.11}{do not meet the assumptions in \cite{grizzle2014models}} (e.g., solutions with grazing or with a change in phase sequence). 
We refer to these points as inadmissible points~(IP) (Fig. \ref{fig:Generators}).
Alternatively, they can border to a point $(\vec{x}_0, T, \bar{E})\in \mathcal{G}$ for which $\mat{R}_{\bar{E}}$ becomes rank deficient.
These singularities either constitute turning points~(TP) (i.e., extremal values for the parameter~$\bar{E}$)
or bifurcations~(BP) in which the periodic solutions of equation~\eqref{eq:uniqueSolution} are no longer distinct.
Both types of singularities connect different generators to form a connected component~$\mathcal{V}$ of the gait space.

\begin{definition}[Connected Components]
    A set $\mathcal{V}\subseteq~\mathcal{G}$ is a connected component of $\mathcal{G}$ if $\mathcal{V}$ is path-connected and is maximal with respect to inclusion (Definition 2.3 \cite{Rosa2021}).
\end{definition}

How such connected components of the gait space can be efficiently computed, will be discussed in the following.

%
\section{Implementation}
\label{sec:implementation}
\subsection{Constructing Conservative Models}
To implement an energetically conservative model of legged locomotion, the properties \textit{Df1} and \textit{Df2} in \revcomment{2.2}{the ECM definition} must be fulfilled.
\revcomment{1.11}{\textit{Df1} can be easily satisfied by implementing ideal constraints and omitting additional joint torques in equations \eqref{eq:DAE}}.
To satisfy \revcomment{1.11}{\textit{Df2}} at touch-down events, we have to account for the changes in velocity, yielding:
\begin{equation}
    \dot{\vec{q}}^{+{\mathop{\mathrm{T}}}} \vec M \dot{\vec{q}}^+ - \dot{\vec{q}}^{-{\mathop{\mathrm{T}}}} \vec M \dot{\vec{q}}^- =0, \quad \forall \dot{\vec{q}}^-.\label{eq:constantKineticEnergy}
\end{equation}
Using the projection in equation \eqref{eq:DiscreteMap}, we can write this as
\revcomment{1.11}{
\begin{equation}\label{eq:noEnergyLoss}
    \left(\mat W_{i+1}\T \dot{\vec{q}}^-\right)\T \mat G_{i+1}^{-1} \left(\mat W_{i+1}\T \dot{\vec{q}}^-\right) = 0,~~ \forall \left(\mat W_{i+1}\T \dot{\vec{q}}^-\right).
\end{equation}
}
In the general case, energy conservation would only be possible if the inverse Delassus' matrix \revcomment{1.11}{$\mat G_{i+1}^{-1}$} were zero.
Loosely speaking, this is because inertia and masses involved in the projection need to vanish to conserve energy.
This is problematic, as this requirement leads to singularities in the systems mass matrix $\mat{M}$.

Instead, we consider vanishing masses and inertias only as a limiting case.
That is, with some abuse of notation, we define a parameterized mass matrix $\mat{M}(\vec{q},\varepsilon)=\mat M_\varepsilon$, with parameter~$\varepsilon$ such that the Delassus' matrix reads as \revcomment{1.11}{$\mat{G}_{i+1}(\vec{q},\varepsilon)$}. 
This parameterization must yield
\revcomment{1.11}{
\begin{equation}\label{eq:ZeroDelassus}
    \lim_{\varepsilon\to 0} \mat{G}_{i+1}(\vec{q},\varepsilon)^{-1}=\mat{0}.
\end{equation}
}
Considering equation \eqref{eq:noEnergyLoss}, the mechanical system is only energetically conservative in the limit of $\varepsilon \to 0$.
As pointed out in chapter 2.3. of \cite{johnson2016hybrid}, massless appendages of a robot possibly yield an inconsistent relationship between accelerations and net forces in equation \eqref{eq:DAE1}. Hence, any rank deficiency of the mass matrix $\mat M_{\varepsilon=0}$ has to be corrected by constraints~\mbox{($\mat{W}_i$, $\vec{\lambda}_i$)} to ensure unique, finite dimensional dynamics.
With this, it is possible to cancel out appearing singularities in the inverse mass matrix $\mat M_\varepsilon^{-1}$ by introducing a parametric scaling with $\varepsilon$ in $\vec{h}$, $\vec{k}$ and $\vec{\lambda}_i$ such that equation~\eqref{eq:DAE1} can be stated as
\revcomment{3.3}{
\begin{equation}\label{eq:scaledEOM}
    \ddot{\vec{q}} = \mat M_\varepsilon^{-1}\vec{h}(\vec{x},\varepsilon)+\mat M_\varepsilon^{-1}\left(\vec{k}(\vec{q},\varepsilon)+\mat W_i(\vec q) \vec \lambda_i(\varepsilon)\right).
\end{equation}}
The resulting conservative vector field, defined by equation~\eqref{eq:scaledEOM}, is \revcomment{1.11}{$C^1$} and complete in the analytic limit of $\varepsilon\to 0$.
\revcomment{3.3}{In other words: while $\mat M_\varepsilon$ can become singular in the limit of $\varepsilon \to 0$, the products $\mat M_\varepsilon^{-1}\vec{h}$ and $\mat M_\varepsilon^{-1}(\vec{k} + \mat W_i \vec \lambda_i)$ remain finite.}
\begin{remark}
The vector field properties are similar to A7 in \cite{johnson2016hybrid}.
However, we do not \revcomment{3.3}{need to require} $\left[\begin{smallmatrix} \mat M&\mat{W}_i\\ \mat{W}_i\T&\vec 0 \end{smallmatrix}\right]$ to be invertible in the limit and do not explicitly change the topology of the robot whenever a massless limb is unconstrained to the ground (A6 in \cite{johnson2016hybrid}).
\end{remark}

\subsection{Numerical Exploration}
The goal of our implementation is to solve the implicit function~\eqref{eq:uniqueSolution} in a systematic fashion to obtain the connected component $\mathcal{V}$. 
Our primary tool for the computation of generators are numerical continuation methods \cite{allgower2012}.

The issue with numerically solving equation~\eqref{eq:uniqueSolution} is that it has $2n_\mr{q}+2$ constraints but only $2n_\mr{q}+1$ decision variables in~$\vec{x}_0$ and $T$. 
In theory, this is no problem, as the equations in~\eqref{eq:uniqueSolution} are not independent due to the energetically conservative nature of the dynamics \cite{pokorny2009}, as was shown above.
In practice, however, this can cause issues, as fluctuations in energy can be introduced during numerical integration.
When this is the case, equation~\eqref{eq:uniqueSolution} may not be solvable with only $2n_\mr{q}+1$ decision variables.
To tackle this issue, we use the approach reported in \cite{sepulchre1997} and add a parameter $\xi$ to the continuous dynamics \revcomment{1.11}{\eqref{eq:DAE}:
\begin{equation}\label{eq:hybridDynamicsNum}
\mathcal{F}_\xi = {\{\tilde{\vec{f}}_i\}}_{i=1}^m : \tilde{\vec{f}}_i := \vec f_i(\vec x)+ \xi\cdot \nabla E(\vec x).
\end{equation}
}
\revcomment{1.11}{With the new representation \eqref{eq:hybridDynamicsNum}, the conservative system~$\Sigma$ is embedded in a one-parameter family of dissipative dynamics~\mbox{$\Sigma_\xi =\left(\mathcal{X},\mathcal{E},\mathcal{D},\mathcal{F}_\xi\right)$}.} 
Analytically, a periodic orbit only exists for a vanishing perturbation $\xi$ (Lemma 1 \cite{sepulchre1997}). 
Hence, solutions \mbox{$\vec\varphi(t,\vec x_0, \xi)$} of $\Sigma_{\xi}$ with $\xi=0$ are periodic solutions of the underlying conservative system.
In the numerical computation of gaits, however, we might obtain solutions with a small~$\xi$ to compensate for small energy losses caused by numerical damping in the integration schemes.

Gaits of legged systems, are not necessarily periodic in all states.
In particular, the horizontal position is aperiodic to allow for forward motion.
Hence, to relax the periodicity constraint \eqref{eq:periodicity}, we split the state $\vec{x}$ into a periodic part \mbox{$\vec{x}_{\mr{p}}:=\mat A_{\mr{p}}\vec x$} and a non-periodic part \mbox{$\vec{x}_{\mr{np}}:=\mat A_{\mr{np}}\vec x$}  by introducing the constant orthonormal selection matrix \mbox{$\mat A_{\mr{s}}=\left[\begin{smallmatrix} \mat A_{\mr{p}}\\ \mat A_{\mr{np}}\end{smallmatrix}\right]\in \mathbb{R}^{2n_\mr{q}\times 2n_\mr{q}}$}. 

%
\revcomment{1.11}{In the following, we do not implement the time-to-impact function and thus, decouple the time duration~$t_i$ of each phase~$i$ from the initial conditions~$\vec x_{0,i}$.} 
This allows us to move away from an event-driven evaluation of~$\Sigma_{\xi}$.
In this approach, the event constraints \revcomment{1.11}{$e_i^{i+1}$} become explicit components of the root function~$\vec{r}_{\bar{E}}$, rather than being implicitly stated in the set~\revcomment{1.11}{$\mathcal{E}_i^{i+1}$}.
This change greatly facilitates the computation of the derivatives in $\mat{\Phi}_T$.
Hence, a periodic solution for a given~$\bar{E}$ can be obtained numerically by solving the root-finding problem~\revcomment{1.11}{\mbox{$\tilde{\vec{r}}_{\bar{E}}:\mathbb{R}^{2n_\mr{q}+m+2}\to\mathbb{R}^{2n_\mr{q}+m+2}$}:
\begin{equation}\label{eq:rootNum}
    \begin{aligned}
    \tilde{\vec r}_{\bar{E}}(\vec x_0,\vec{t},\xi)&= \begin{bmatrix}[1.2]  \mat{A}_{\mr{p}}\cdot\left(\vec{\varphi}_{1}(t_{m+1},\vec{x}_{0,m+1};\xi)-\vec{x}_0\right)\\
    \mat{A}_{\mr{np}}\cdot \vec{x}_0\\
    a(\vec x_0)\\
    E(\vec x_0) -\bar{E}\\
    e_{m}^{1}\left(\vec{\varphi}_m\left(t_{m},\vec{x}_{0,m};\xi\right)\right)\\
    \vdots\\
    e_{1}^{2}\left(\vec{\varphi}_1\left(t_1,\vec{x}_{0,1};\xi\right)\right)
    \end{bmatrix}\\[1mm]
    &=\vec 0,
    \end{aligned}
\end{equation}
where $\vec{t}=[t_1\dots t_{m+1}]\T$ and the initial states $\vec x_{0,i}$ of each mode are defined recursively as in equations \eqref{eq:recursion}, substituting the function $t_{\mr{I},i}$ by the variable $t_i$.}
With \mbox{$\vec{z}\T=[\vec{x}_0\T~\vec t\T~\xi]$}, we refer to the Jacobian of $\tilde{\vec r}_{\bar{E}}$ as $ \tilde{\mat R}_{\bar{E}} := \partial \tilde{\vec{r}}_{\bar{E}}/\partial\vec{z}$.

In addition to the implicit equation \eqref{eq:rootNum}, we define an extended root function \revcomment{1.11}{\mbox{$\tilde{\vec r}:\mathbb{R}^{2n_\mr{q}+m+3}\to\mathbb{R}^{2n_\mr{q}+m+2}$}} that also includes $\bar{E}$ as a free variable:
\begin{align}\label{eq:rootNumMani}
    \tilde{\vec r}(\underbrace{\vec z,\bar{E}}_{=:\vec{u}}) &:= \tilde{\vec r}_{\bar{E}}(\vec z),\\
    \tilde{\mat R}(\vec{u}) &:= \dfrac{\partial\tilde{\vec r}}{\partial \vec u} = \begin{bmatrix}\tilde{\mat R}_{\bar{E}}(\vec{z}) & \dfrac{\partial\tilde{\vec r}}{\partial \bar{E}}\end{bmatrix}.
\end{align}
If $\vec{z}^\ast$ is a regular point of $\tilde{\vec{r}}_{\bar{E}}$, then $\tilde{\vec{r}}(\vec{u)}=\vec{0}$ characterizes a locally defined 1D solution manifold. 
The function $\tilde{\vec{r}}$ is well suited for a pseudo-arclength continuation which is utilized to compute generators. 
This approach employs a predictor-corrector (PC) method with a variable step size~$h$ (Chapter~6.1 \cite{allgower2012}), which takes small iterative steps in the tangent space of $\tilde{\vec r}(\vec{u})=\vec 0$ to locally trace the solution curve of regular points.
This tangent space is equivalent to the kernel of $\tilde{\mat{R}}$ at a regular point $\vec{u}^\ast$ of equation \eqref{eq:rootNumMani}, with the tangent vector $\vec{p}$:
\begin{align}
    \tilde{\mat{R}}(\vec{u}^\ast)\vec{p}=\vec 0,~~
    \Vert\vec p\Vert_2 = 1,~~  \det \Bigg(\underbrace{\begin{bmatrix}[1.2]\tilde{\mat{R}}(\vec{u}^\ast)\\
    \vec p\T\end{bmatrix}}_{=:\mat{J}}\Bigg)>0. \label{eq:testBP}
\end{align}
\begin{algorithm}[t]
\small
\LinesNumbered
\DontPrintSemicolon 
\KwIn{Regular point $\vec{u}^\ast$; Initial step-size $h>0$ \mbox{Maximal number of generated points $N_\mathrm{max}$}}
\KwOut{Generator $\mathcal{M}_j$, BP, TP, IP}
$\vec u^{0} \gets \vec{u}^\ast$\;
add $\vec u^0$ to $\mathcal{M}_j$\;
$d_{\mr{t}} = +1$ \Comment*[r]{Direction of curve}
\While{$k = 0\dots N_\mathrm{max}$ } {
  \SetKwProg{Fn}{PC-step}{:}{}
  \Fn{$\mr{(}\vec{u}^k,~d_{\mr{t}}\mr{)}$}{
        \textbf{Predictor Step} (Explicit-Euler Step)\;
        $\vec{u}_{\mr{pred}}^{k+1}\gets \vec{u}^{k}+d_\mr{t} h \vec p^k$\;
        \textbf{Corrector Step} (Newton's Method)\;
        \KwRet $\vec u^{k+1}=(\vec{z}^{k+1},\bar{E}^{k+1})$ \;
  }
  $isSpecialPoint \gets true$\;
  \uIf{$\vec{u}^{k+1}$ is inadmissible}{
    search for IP between $\vec{u}^k$ and $\vec{u}^{k+1}$\;
  }
  \uElseIf{$\vec{p}^k\cdot\vec{p}^{k+1}<0$}{
    search for simple BP between $\vec{u}^k$ and $\vec{u}^{k+1}$ \;
  }
  \uElseIf{$\det(\tilde{\mat{R}}_{\bar{E}}(\vec{z}^{k}))\cdot\det(\tilde{\mat{R}}_{\bar{E}}(\vec{z}^{k+1}))<0$}{
    search for TP between $\vec{u}^k$ and $\vec{u}^{k+1}$\;
  }
  \Else{
  add $\vec u^{k+1}$ to $\mathcal{M}_j$\;
  $isSpecialPoint \gets false$\;
  }
  \If{$isSpecialPoint$}{
  \eIf{$d_\mr{t}=+1$}{$d_\mr{t}=-1$\;
  $\vec{u}^{k+1}\gets\vec{u}^0$}
  {$\textbf{break}$}
  }
}
\Return{$\mathcal{M}_j$, $\mr{BP}$, $\mr{TP}$, $\mr{IP}$}\;
\caption{Compute Generator $\mathcal{M}_j$}
\label{algo:computeGenerator}
\end{algorithm}
\begin{algorithm}[t]
\small
\LinesNumbered
\DontPrintSemicolon 
\KwIn{Starting point $\vec{u}_0$; \mbox{Maximal number of generators $N_\mathrm{max}$}}
\KwOut{Connected Component $\mathcal{V}$}
push $\vec{u}_0$ to queue $\mr{Q}$\;
\While{$k = 1\dots N_\mathrm{max}$ $\bf{and}$ $\mr{Q}$ is not empty} 
{
  \SetKwProg{Fn}{Algorithm 1}{:}{}
  pull $\vec{u}^\ast$ from Q\;
  \Fn{$\mr{(}\vec{u}^\ast\mr{)}$}{
        \KwRet $\mathcal{M}_k$, TP, BP, IP\;
  }
  add $\mathcal{M}_k$, TP, BP to $\mathcal{V}$\;
  find regular points $\vec{u}_i^\ast$ in nbhd of TP, BP\;
  \lForEach{$\vec{u}_i^\ast$ not in $\mathcal{V}$}{push $\vec{u}_i^\ast$ to $\mr{Q}$}}
\Return{$\mathcal{V}$}\;
\caption{Compute Connected Component $\mathcal{V}$}
\label{algo:Strat}
\end{algorithm}
As the curve can be locally pursued in two directions, $\det(\mat J)>0$ defines positive orientation \cite{allgower2012}.

In this process, the crossing of simple (codimension-one\footnote{A simple or codimension-one bifurcation point $\vec{u}_{s}$ is defined by a loss of rank in $\tilde{\mat{R}}$, i.e., $\mr{rank}(\tilde{\mat{R}}(\vec{u}_{s}))= 2n_\mr{q}+n$.}) bifurcations are detected by a flip in direction of the tangent vector $\vec{p}$ (i.e., $\vec{p}^k \cdot \vec{p}^{k+1}<0$) \cite{allgower2012}. 
The detection of turning points (TP) follows from a change in sign of $\det(\tilde{\mat{R}}_{\bar{E}}(\vec{z}))$ (i.e., \mbox{$\det(\tilde{\mat{R}}_{\bar{E}}(\vec{z}^{k}))\det(\tilde{\mat{R}}_{\bar{E}}(\vec{z}^{k+1}))<0$}), in which $\vec u$ remains a regular point of equation \eqref{eq:rootNumMani}.
In Algorithm~\ref{algo:computeGenerator}, the curve is traversed in both directions until a \emph{special} point is detected. 
Special points $\vec{u}^{k+1}$ are the result of a \mbox{PC-step} that has crossed a BP, TP, or IP. Herein, non-successful PC-steps (e.g., divergence in Newton's method) are also considered inadmissible (IP).
The algorithm returns the new generator~$\mathcal{M}_j$ and its associated TPs and BPs.
The curve~$\mathcal{M}_j$ has at most 2 limiting special points.
As mentioned previously, TPs and BPs are singular points that connect to different generators $\mathcal{M}_j$. 
\mbox{Algorithm \ref{algo:Strat}} constructs a subset of the space of connected components. 
It utilizes a breadth-first-search to explore different generators given the location of connected TPs and BPs.
Locations of regular points $\tilde{\vec{u}}_i$ in the neighborhood of simple bifurcations can be found with the bifurcation equation~(Chapter~8.3~\cite{allgower2012}).
As indicated above, it is essential to have a problem specific starting point $\vec{u}_0$ that solves equation \eqref{eq:rootNum} and is regular.

We note that Algorithm \ref{algo:computeGenerator} is only able to detect TPs and simple BPs. 
Bifurcations of codimension-two and higher are overlooked or wrongly classified as simple bifurcations.
Test functions for their detection are described in \cite{beyn2001}.
\section{Example: One-Legged Hopper}
\label{sec:example}
\subsection{Model Description}
In this section, we highlight the application of our method to a SLIP-like one-legged hopper introduced in \cite{gan2018} with passive swing leg dynamics that are created by a torsional hip spring~(Fig. \ref{fig:Hopper}).
Here, however, it is derived in a more formal manner including a rigorous treatment of the previously unsolved issue of the spring leg dynamics during flight.
This motion, which becomes singular for vanishing foot-masses, was simply ignored in \cite{gan2018} and is treated here by the inclusion of additional holonomic constraints.

The model consists of a torso with mass $m_\mathrm{t}$ which is constrained to purely linear motions as defined in \cite{gan2018}. 
Thus, the torso's configuration is given by the hip position \mbox{($x, y$)}.
The leg is connected to the hip via a rotational joint (with joint angle~$\alpha$) that includes a torsional spring (with stiffness~$k_\alpha$ and no damping).
We model the legs as massless linear springs with leg length~$l$, natural spring length~$l_o$, spring stiffness~$k_\mathrm{l}$, no damping, and a point mass~$m_\mathrm{f}$ at the foot.  
The total mass of the model is $m_o=m_t+m_f$.
We use generalized coordinates~$\vec{q}=[x~y~\alpha~l]\T$ (i.e., \mbox{$n_\mathrm{q}=4$}) to represent the configuration of the robot. 

The model has two \revcomment{1.11}{phases}: \mbox{\textit{stance} $\mr{S}$} and \mbox{\textit{flight} $\mr{F}$}. 
The corresponding constraint forces in these phases are \mbox{$\vec{\lambda}_\mr{S}=[\lambda_\mr{T}~\lambda_\mr{N}]\T$} and $\lambda_\mathrm{F}$. 
These forces satisfy the constraints
\begin{alignat}{2}
    g_\mathrm{F}(\vec{q}) &= l-l_o&=0,\label{eq:ConstraintF}\\
    \vec{g}_\mathrm{S}(\vec{q}) &= \begin{bmatrix}x+l\sin(\alpha)-x_\mathrm{c}\\y-l\cos(\alpha) \end{bmatrix}&= 0,\label{eq:ConstraintS}
\end{alignat}
during \textit{flight} and \textit{stance}, respectively.
The constraint \eqref{eq:ConstraintF} fixes the leg length to $l_o$ during flight, whereas equation \eqref{eq:ConstraintS} implements the assumption of no sliding during \textit{stance} (with a horizontal contact point position $x_\mathrm{c}$).
%
For the continuous dynamics in equations \eqref{eq:DAE}, we have
\begin{equation}
    \vec k\T = \begin{bmatrix}0&0&F_\alpha &F_\mathrm{l} \end{bmatrix},
\end{equation}
where \mbox{$F_\alpha(\vec q) = - k_\alpha \alpha$},
 \mbox{$F_\mathrm{l}(\vec q) = k_\mathrm{l}(l_o-l)$},
describe the hip and leg spring forces, respectively.
Note, the tangential and normal contact forces $\lambda_\mathrm{T}$, $\lambda_\mathrm{N}$ are only active during \textit{stance}. Similar, $\lambda_\mathrm{F}\neq 0$ only holds during \textit{flight} to constraint the leg length to its natural length $l_o$. This leads to impulsive forces and thus discontinuous changes in $\dot l$ whenever the foot leaves the ground with non-zero velocity. The touch-down event \revcomment{1.11}{$e_{\mr{F}}^{\mr{S}}(\vec q) = [0~1]\cdot\vec{g}_\mathrm{S}(\vec q)$}
is defined kinematically, while the lift-off event \revcomment{1.11}{$e_{\mr{S}}^{\mr{F}}(\vec q,\dot{\vec q}) = \lambda_\mathrm{N}$} is triggered when $\lambda_\mathrm{N}$ changes sign from positive to negative.
\revcomment{1.11}{We restrict all motions to the cycle $\mr{F}\to\mr{S}\to\mr{F}$}, which is started at apex transit $a(\vec x_0) = \dot{y}_0$ during \textit{flight}. \par
\begin{figure}[t]
    \centering
    \includegraphics{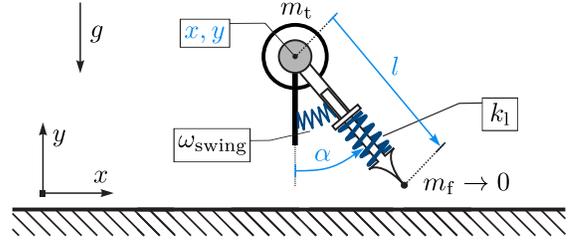}
    \caption{An energetically conservative one-legged hopper with a torsional hip spring. It's planar configuration is described by $\vec{q}=[x~y~\alpha~l]\T$.
}
    \label{fig:Hopper}
\end{figure}
Eventually, we would like to bring the foot mass $m_\mathrm{f}$ to zero to avoid kinetic energy losses during touch-down, similar to the method used in \cite{gan2018, oconnor2009, garcia1998simplest, johnson2016hybrid}. 
To fulfill \textit{Df2} and satisfy equation \eqref{eq:ZeroDelassus}, we redefine the foot mass by \mbox{$m_\mathrm{f}=\varepsilon \hat{m}_\mathrm{f}$}.
Note, for $\varepsilon\to 0$, the condition in equation \eqref{eq:ZeroDelassus}, is equally satisfied for the \textit{stance} and \textit{flight} transition:
\begin{alignat}{2}
    \lim_{\varepsilon\to 0} \mat G_\mathrm{S}(\vec{q},\varepsilon)^{-1} &= \lim_{\varepsilon\to 0} \begin{bmatrix}[1.2]\varepsilon \hat{m}_\mathrm{f} & 0\\ 0 &  \varepsilon \hat{m}_\mathrm{f}\end{bmatrix} &&=  \mat 0,\\
    \lim_{\varepsilon\to 0} G_\mathrm{F}(\vec{q},\varepsilon)^{-1} &= \lim_{\varepsilon\to 0} \dfrac{\varepsilon \hat{m}_\mathrm{f} m_\mathrm{t}}{\varepsilon \hat{m}_\mathrm{f}+m_\mathrm{t}} &&=  0.
\end{alignat}
Further, to maintain finite continuous dynamics \eqref{eq:DAE1} in the limit $\varepsilon\to 0$, we redefine the constraint forces as:
\begin{gather}\label{eq:lambda_new}
    \lambda_\mathrm{F}(\varepsilon,\vec q) = \frac{\varepsilon \hat{m}_\mathrm{f}}{m_o}\hat{\lambda}_\mathrm{F}-F_\mathrm{l}, \quad\vec{\lambda}_\mathrm{S}(\varepsilon, \vec q) = \frac{\varepsilon\hat{m}_\mathrm{f}}{m_o}\hat{\vec{\lambda}}_\mathrm{S}+\vec{s}F_\mathrm{l},
\end{gather}
with \mbox{$\vec{s}(\vec{q})=\left[\begin{smallmatrix} -\sin(\alpha)\\\cos(\alpha) \end{smallmatrix}\right]$}, introducing new auxiliary forces $\hat{\lambda}_\mathrm{F}$, $\hat{\vec{\lambda}}_\mathrm{S}$.
The core idea here is to separate the constraint forces into two components, where the first balances the elastic forces which are expressed by the known values of $F_\mathrm{l}$.
The second component balances the inertial forces and is computed when solving the differential-algebraic equations \eqref{eq:DAE}. 
This second component is further scaled with $\varepsilon$ to yield finite values for~$\hat{\lambda}_\mathrm{F}$ and $\hat{\vec{\lambda}}_\mathrm{S}$, even in the limit \mbox{$\varepsilon\to 0$}.
Equivalently to \cite{gan2018}, we prescribed a leg swing frequency $\omega_\mathrm{swing}$ by the relation 
\begin{equation}\label{eq:legswingfreq}
    k_\alpha = \omega_\mathrm{swing}^2 \underbrace{m_\mathrm{f}}_{=\varepsilon \hat{m}_\mathrm{f}}l_o^2.
\end{equation}
This implies that $\omega_\mathrm{swing}$ remains a finite constant value when the foot mass $m_\mathrm{f}$ is brought to zero and thus \mbox{$k_\alpha \to 0$}.
With the modifications in equations \eqref{eq:lambda_new}, \eqref{eq:legswingfreq} and taking the limit~\mbox{$\varepsilon\to 0$}, we arrive at the same finite dimensional dynamics reported in \cite{gan2018}.
To allow for horizontal displacement in equation \eqref{eq:rootNum}, the matrix $\mat{A}_\mr{np}$ selects the initial state~\mbox{$x_0=\vec{A}_\mr{np}\vec x_0$}.
The remaining periodic states are selected by its orthogonal complement $\mat{A}_\mr{p}$.
In this energetically conservative model, all state and parameter values are normalized with respect to~$m_o$, $g$ and $l_o$.
To allow a comparison with \cite{gan2018}, we set the leg stiffness to \unitfrac[$k_\mr{l}=40$]{$m_og$}{$l_o$} (which is equivalent to hopping with 2 legs of stiffness \unitfrac[$20$]{$m_og$}{$l_o$}) and the swing frequency to~$\omega_\mathrm{swing}=\sqrt{\unitfrac[5]{g}{l_o}}$.

\subsection{Results}
\begin{figure}[t]
    \centering
    \includegraphics[width=1\columnwidth]{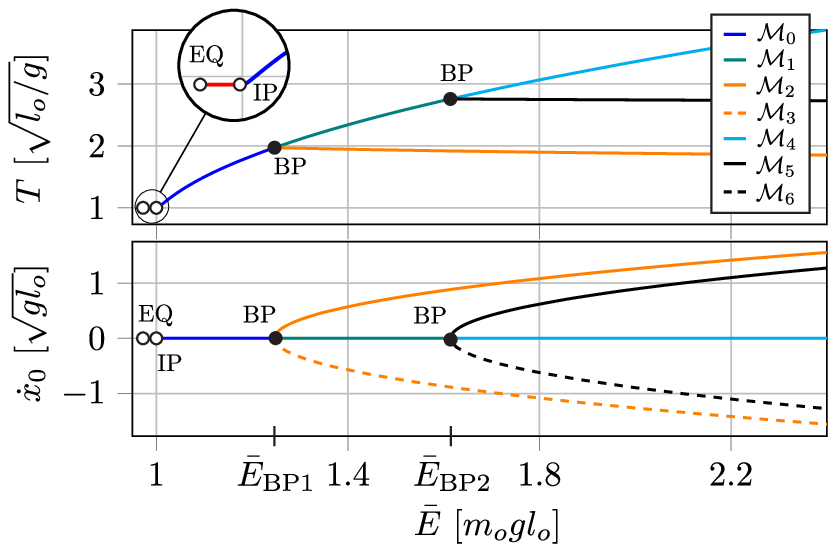}
    \caption{Visualization of connected generators $\mathcal{M}_0$ - $\mathcal{M}_6$ of the one-legged hopper. 
    Vertical hopping in-place motions are contained in $\mathcal{M}_0$, $\mathcal{M}_1$ and $\mathcal{M}_4$, while $\mathcal{M}_2$, $\mathcal{M}_5$ and $\mathcal{M}_3$, $\mathcal{M}_6$ are a collection of forward, backward gaits, respectively. All generators constitute the connected component $\mathcal{V}$ since they are connected by simple bifurcation points (BP). 
    The equilibrium~(EQ) and the contact sequence transition with vanishing \textit{flight} duration are inadmissible points~(IP).
    The locally defined 1D manifold of linear bouncing-in-place oscillations (red) is thus not in $\mathcal{V}$.
    }
    \label{fig:GeneratorSLIP}
\end{figure}
Using this model, Algorithm 2 was initialized with a vertical hopping motion at energy level \mbox{$\bar{E}=1.001~m_ogl_o$} (that is, with initial apex height of $y_0=1.001~l_o$).
Here, the motion in~$y$ and $l$ simply follows a parabolic trajectory during \textit{flight} and a linear oscillation during \textit{stance}.
There is no movement in~$x$ and $\alpha$.
This hopping motion constitutes a regular point~$\vec{u}_0$ that solves equation \eqref{eq:rootNum}. 
This initial point is connected to a locally defined \mbox{1D} manifold $\mathcal{M}_0$ (Fig. \ref{fig:GeneratorSLIP}) of hopping in place motions.
Towards lower energies, hopping height is reduced and this generator is bounded by a point that corresponds to a vanishing time $t_\mr{F}$ in \textit{flight} at energy level $\bar{E}=1~m_ogl_o$.
Periodic solutions of $\Sigma$ with even lower energy do exist, yet they correspond to an oscillating in-place motion.
Since there is no lift-off in this motion, going beyond this point leads to a change in phase sequence.
\revcomment{1.11}{This is an inadmissible point}
However, there exists a locally defined manifold with this different contact sequence~\revcomment{1.11}{$\mr{S}\to\mr{S}$}\footnote{Of interest is its connectedness to an equilibrium point (EQ) (Fig. \ref{fig:GeneratorSLIP}).
The phase sequence~\revcomment{1.11}{$\mr{S}\to\mr{S}$} admits solutions in the linear eigenspace of a 1D oscillator.
These linear modes exist in the range \mbox{$\bar{E}\in(\bar{E}_\mr{EQ},m_ogl_o)$}, where $\bar{E}_\mr{EQ}=m_og(l_o-m_og/k)$ is the energy at EQ.}.
It can be independently computed by Algorithm~\ref{algo:computeGenerator}, however, it is not in the connected component~$\mathcal{V}$ of generators with contact sequence \revcomment{1.11}{\mbox{$\mr{F}\to\mr{S}\to\mr{F}$}}.

\begin{figure}[t]
    \centering
    \includegraphics{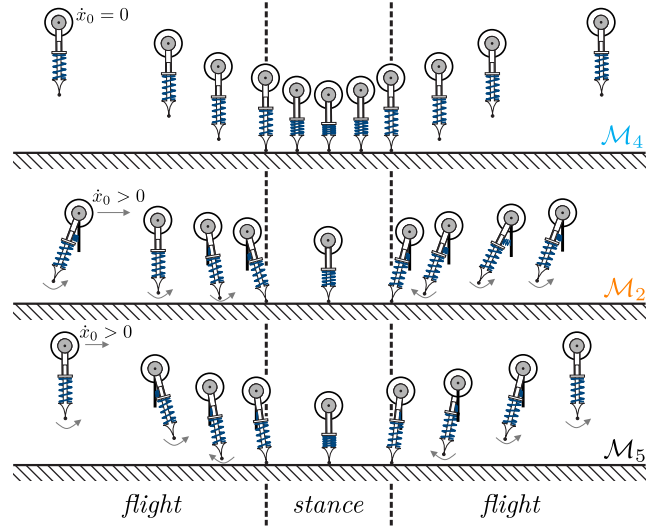}
    \caption{Key frames from periodic solutions of the SLIP model at energy level $\bar{E}=1.8~m_ogl_o$. 
    Starting from apex transit ($\dot{y}_0=0$), three gaits from the generators $\mathcal{M}_4$, $\mathcal{M}_2$ and $\mathcal{M}_5$ are illustrated in the contact sequence $\{\mr{F},\mr{S},\mr{F}\}$.
    The \textit{stance} duration differs between these gaits with: \mbox{$t_\mr{S}^{\mathcal{M}_4}\approx 0.54~\sqrt{l_o/g}$}, \mbox{$t_\mr{S}^{\mathcal{M}_2}\approx 0.49~\sqrt{l_o/g}$}, \mbox{$t_\mr{S}^{\mathcal{M}_5}\approx 0.53~\sqrt{l_o/g}$}.
    }
    \label{fig:SLIPsequence}
\end{figure}

Carrying on with~\revcomment{1.11}{$\mr{F}\to\mr{S}\to\mr{F}$}, we traverse the generator~$\mathcal{M}_0$ towards higher energies. 
$\mathcal{M}_0$ is bounded by a simple bifurcation point at energy level \mbox{$\bar{E}_{\mr{BP1}}\approx 1.247~m_ogl_o$}.  
At this point, we find three nearby generators for which the last tangent direction $\vec p$ of $\mathcal{M}_0$ points into the new generator~$\mathcal{M}_1$. 
$\mathcal{M}_1$ consists of purely vertical hopping motions at higher energies than in $\mathcal{M}_0$.  
The remaining generators, $\mathcal{M}_2$ and $\mathcal{M}_3$, consist of forward \mbox{($\dot{x}_0>0$)} and backward \mbox{($\dot{x}_0<0$)} hopping motions, respectively.
The computation of $\mathcal{M}_1$ leads to another simple bifurcation point at $\bar{E}_{\mr{BP2}}\approx 1.614~m_ogl_o$. 
We find three new connected generators $\mathcal{M}_4$--$\mathcal{M}_6$. 
The vertical motions in $\mathcal{M}_4$ are similar to gaits in $\mathcal{M}_0$ and $\mathcal{M}_1$.  
The generators $\mathcal{M}_5$ and $\mathcal{M}_6$ correspond to forward and backwards hopping motions, respectively.

Figure \ref{fig:SLIPsequence} illustrates three gaits from $\mathcal{M}_4$, $\mathcal{M}_2$ and $\mathcal{M}_5$ at energy level $\bar{E}= 1.8~m_ogl_o$.  The motions in $\mathcal{M}_2$ and $\mathcal{M}_5$ are qualitatively different in the leg's angular velocity at touch-down. 
In $\mathcal{M}_2$, the foot touches down with \revcomment{1.11}{$\dot{\alpha}_{\mr{td}}>0$} while the gaits in $\mathcal{M}_5$ possess a longer \textit{flight} duration~$t_\mr{F}$ in which the foot touches down in a returning motion with \revcomment{1.11}{$\dot{\alpha}_{\mr{td}}<0$} (so called \emph{speed matching}).  
This holds equivalently for backward hopping in  $\mathcal{M}_3$ and $\mathcal{M}_6$.


We stopped the exploration of $\mathcal{M}_2$--$\mathcal{M}_6$ with regular points at $\bar{E}=2.4~m_ogl_o$.
It is possible to encounter more \textit{special} points (BP, TP, IP) in the numerical continuation at higher energy levels.
It took approximately a minute on a laptop with an i5-8265U CPU @1.60GHz and 4GB RAM to generate the data\footnote{The code to generate this data can be found at \href{https://github.com/raffmax/ConnectingGaitsinEnergeticallyConservativeLeggedSystems}{https://github.com/ \\raffmax/ConnectingGaitsinEnergeticallyConservativeLeggedSystems}} presented in Figures \ref{fig:GeneratorSLIP} and \ref{fig:SLIPsequence}.

%
\section{Discussion \& Conclusion}
\label{sec:Discus}

In this paper, we introduced a formal framework and a generalized methodology for the computation of connected gaits in energetically conservative legged systems.
This work extends and clarifies the methodology introduced in \cite{gan2018} to apply not only to the gaits of legged models but to a broader class of ECMs.
In terms of theory, our work extends the results in \cite{sepulchre1997} to hybrid dynamical systems and clarifies the connected structure of the gait space $\mathcal{G}$ of energetically conservative legged systems.
%

Our contributions further relate the study of passive gaits to established and emerging concepts in the field of nonlinear dynamics.  
Similar to the generators in \cite{albu2020}, we (locally) define 1D manifolds in which there is a unique relation between motion and energy.
However, our definition of these generators is different in that these 1D manifolds do not include equilibria and they are defined for hybrid dynamical systems.
As a consequence, the direct connection to linear oscillations, that occur in the linearized system at equilibrium and that is a characteristic of today's NNMs definitions, is lost.
This loss is caused by two required assumptions.
The first is the transversality condition of the anchor constraint that is violated in an equilibrium.  
We introduced it here to impose a Poincaré section, yet it can potentially be lifted, as it is done in \cite{Rosa2021}.
\revcomment{1.11}{The second is the fixed phase sequence that prohibits the connection of an equilibrium at standstill to a forward gait.} 

As shown in Fig.~\ref{fig:GeneratorSLIP}, the linear modes of the 1D oscillator correspond to bouncing in place.
In future work, it may be possible to formally link them to the hopping gaits characterized in this paper.
To make this possible, we need to relax the requirement that the \revcomment{1.11}{phase} sequence is fixed.
This assumption constitutes the primary limitation of our work.
It is necessary, as the core results in this paper follow from the monodromy matrix $\mat{\Phi}_T$. 
For a fixed \revcomment{1.11}{phase} sequence, $\mat{\Phi}_T$ changes differentiably in neighboring periodic solutions and so does the associated tangent space.
\revcomment{1.11}{This is no longer true when certain assumptions from \cite{grizzle2014models}, e.g., no grazing contacts, do not hold.} 
\revcomment{1.11}{With a vanishing phase duration, a Saltation matrix \cite{leine2013dynamics} may become discontinuous \cite{ivanov1998}, which directly propagates to discontinuities in $\mat{\Phi}_T$ and the associated tangent space.}
For legged systems, continuity in the Saltation matrix can be ensured under certain conditions~\cite{pace2017}.
Turning these conditions into systematic modeling guidelines, or finding ways to connect gaits despite these discontinuities are avenues for future work.


While the focus of this paper is on energetically conservative systems, real robot systems are not energetically conservative and sources of energy loss (heat, impacts, batteries, vibrations) cannot be completely eliminated.  The benefit of our approach is in utilizing the explanatory power of ECMs.  While these systems do not exist in the real world, these simple models often form the core model dynamics for trajectory generation, motion planning, and control algorithms in the field.  
Mapping trajectories from ECMs to more realistic models with energy loss would be an interesting extension of our work, as their passivity makes them ideal candidates for the use as templates to develop energetically economical motions for legged robotic systems.

Beyond this very practical significance, the identified passive motions are a key characteristic of a given ECM.
Their study, not only in simple models of legged systems, will thus allow us to better understand the fundamental nature of gait for both, robotics and biology.

\bibliographystyle{IEEEtran}
\bibliography{references}

\end{document}